\documentclass{article}

\usepackage[utf8]{inputenc} 
\usepackage[T1]{fontenc}    
\usepackage{hyperref}       
\usepackage{url}            
\usepackage{booktabs}       
\usepackage{amsfonts}       
\usepackage{nicefrac}       
\usepackage{microtype}      
\usepackage{amsfonts}
\usepackage{amsmath}
\usepackage{amsthm}
\usepackage{amssymb}
\usepackage[numbers]{natbib}
\usepackage{graphicx}
\usepackage{titlesec}

\pdfoutput=1

\titlespacing\section{0pt}{12pt plus 4pt minus 2pt}{0pt plus 2pt minus 2pt}
\titlespacing\subsection{0pt}{12pt plus 4pt minus 2pt}{0pt plus 2pt minus 2pt}
\titlespacing\subsubsection{0pt}{12pt plus 4pt minus 2pt}{0pt plus 2pt minus 2pt}

\newcommand{\prob}{\mathbb{P}}
\newcommand{\expect}{\mathbb{E}}
\newtheorem{theorem}{Theorem}
\newtheorem{proposition}{Proposition}

\newtheorem{lemma}[theorem]{Lemma}
\newtheorem{definition}{Definition}
\newtheorem{assumption}{Assumption}
\newtheorem{corollary}{Corollary}[proposition]

\title{Causal Inference with Noisy and Missing Covariates via Matrix Factorization}
\author{
  Nathan Kallus\footnote{Alphabetical order}, Xiaojie Mao\footnotemark[\value{footnote}], Madeleine Udell\footnotemark[\value{footnote}]
  \\
  Cornell University\\
  \texttt{\{kallus, xm77, udell\}@cornell.edu} \\
}
\date{\vspace{-5ex}}

\begin{document}
\maketitle

\begin{abstract}
Valid causal inference in observational studies often requires controlling for confounders.
However, in practice measurements of confounders may be noisy,
and can lead to biased estimates of causal effects.
We show that we can reduce the bias caused by measurement noise
using a large number of noisy measurements of the underlying confounders.
We propose the use of matrix factorization to infer the confounders from noisy covariates,
a flexible and principled framework that adapts to missing values,
accommodates a wide variety of data types,
and can augment a wide variety of causal inference methods.
We bound the error for the induced average treatment effect estimator
and show it is consistent in a linear regression setting,
using Exponential Family Matrix Completion preprocessing.
We demonstrate the effectiveness of the proposed procedure in 
numerical experiments with
both synthetic data and real clinical data.
\end{abstract}

\section{Introduction}

Estimating the causal effect of an intervention is a fundamental goal across many domains.
Examples include evaluating the effectiveness of recommender systems \citep{schnabel2016recommendations}, identifying the effect of therapies on patients' health \citep{connors1996effectiveness} and understanding the impact of compulsory schooling on earnings \citep{angrist1991does}.
However, this task is notoriously difficult in observatonal studies due to the presence of confounders: variables that affect both the intervention and the outcomes. For example, intelligence level can influence both students' decisions regarding whether to go to college, and their earnings later on. Students who choose to go to college may have higher intelligence than those who do not. As a result, the observed increase in earnings associated with attending college is confounded with the effect of intelligence and thus cannot faithfully represent the causal effect of college education.

One standard way to avoid such confounding effect is to control for all confounders \citep{imbens2015causal}.
However, this solution poses practical difficulties.
On the one hand, an exhaustive list of confounders is not known a priori, so investigators usually adjust for a large number of covariates for fear of missing important confounders.
On the other hand, \textit{measurement noise} may abound in the collected data: some confounder measurements may be contaminated with noise (e.g., data recording error),
while other confounders may not be amenable to direct measurements and instead admit only proxy measurements. For example, we may use an IQ test score as a proxy for intelligence.
It is well known that using proxies in place of the true confounders
leads to biased causal effect estimates \citep{frost1979proxy, wickens1972note, wooldridge2015introductory}.
However, we show in a linear regression setting that the bias due to measurement noise can be effectively alleviated by using many proxies for the underlying confounders (Section 2.2).
For example, in addition to IQ test score, we may also use coursework grades and other academic achievements to characterize the intelligence.
Intuitively, using more proxies may allow for a more accurate reconstruction of
the confounder and thus may facilitate more accurate causal inference.
Therefore, collecting a large number of covariates is beneficial for causal inference not only
to avoid confounding effects but also to alleviate bias caused by measurement noise.

Although in the big-data era, collecting myriad covariates is easier than ever before,
it is still challenging to use the collected noisy covariates in causal inference.
On the one hand, data is inevitably contaminated with missing values,
especially when we collect many covariates.
Inaccurate imputation of these missing values may aggravate measurement noise.
Moreover,
missing value imputation can at most gauge the values of noisy
covariates but inferring the latent confounders is the most critical
for accurate causal inference.
On the other hand, the large number of covariates may include heterogeneous data types (e.g., continuous, ordinal, categorical, etc.) that must be handled appropriately
to exploit covariate information.

To address the aforementioned problems, we propose to use low rank matrix factorization as a principled approach to preprocess covariate matrices for causal inference.
This preprocessing step infers the confounders for subsequent causal inference from partially observed noisy covariates. Investigators can thus collect more covariates to control for potential confounders and use more proxy variables to characterize the unmeasured traits of the subjects without being hindered by missing values.
Moreover, matrix factorization preprocessing is a very general framework. It can adapt to a wide variety of data types and it can be seamlessly integrated with many causal inference techniques, e.g., regression adjustment, propensity score reweighting, matching \citep{imbens2015causal}.
Using matrix factorization as a preprocessing step makes the whole procedure modular and enables investigators to take advantage of existing packages for matrix factorization and causal inference.

We rigorously investigate the theoretical implication of the matrix factorization preprocessing with respect to causal effect estimation.
We establish a convergence rate for the induced average treatment effect (ATE) estimator and show its consistency in a linear regression setting with Exponential Family Matrix Completion preprocessing \citep{gunasekar2014exponential}.
In contrast to traditional applications of matrix factorization methods with matrix reconstruction as the end goal, our theoretical analysis validates matrix factorization as a preprocessing step for causal inference.

We further evaluate the effectiveness of our proposed procedure on both synthetic datasets and a clinical dataset involving the mortality of twins born in the USA introduced by Louizos et al. \citep{louizos2017causal}.
We empirically illustrate that matrix factorization can accurately estimate causal effects by effectively inferring the latent confounders from a large number of noisy covariates.
Moreover, matrix factorization preprocessing achieves superior performance with loss functions adapting to the data types. It also works well with many causal inference methods and is robust to the presence of missing values.

{\bf Related work.} Our paper builds upon low rank matrix completion methods that have been successfully applied in many domains to recover data matrices from incomplete and noisy observations \citep{bennett2007netflix,cao2015image,schuler2016discovering}.
These methods are not only computationally efficient but also theoretically sound with provable guarantees \citep{gunasekar2014exponential,candes2009exact,candes2010power,candes2010matrix,recht2011simpler,keshavan2010matrix}.
Moreover, matrix completion methods have been developed to accommodate heterogeneous data types prevalent in empirical studies by using a rich library of loss functions and penalties \citep{udell2016generalized}. Recently, Athey et al. \citep{athey2017matrix} use marix completion methods to impute the unobservable counterfactual outcomes and estimate the ATE for panel data. In constrast, our paper focuses on measurement noise in the covariate matrix.
Measurement noise has been considered in literature for a long time \citep{frost1979proxy,wickens1972note}. Louizos et al. recently \citep{louizos2017causal} propose to use Variational Autoencoder as a heuristic way to recover the latent confounders. Similarly, they also suggest that multiple proxies are important for the confounder recovery. In contrast, matrix factorization methods, despite stronger parametric assumptions, address the problem of missing values simultaneously, require considerably less parameter tuning, and have theoretical justifications.

{\bf Notation.} For two scalars $a, b \in \mathbb{R}$, denote $a \vee b = \max \{a, b\}$ and $a \wedge b = \min\{a, b\}$.
For an positive integer $N$, we use $[N]$ to represent the set $\{1, 2, \dots, N\}$. For a set $\Omega$, $|\Omega|$ is the total number of elements in $\Omega$.
For matrix $X \in \mathbb{R}^{N \times p}$, denote its singular values as $\sigma_1 \ge \sigma_2 \ge \dots \ge \sigma_{N \wedge p} \ge 0$. The spectral norm, nuclear norm, Frobenius norm and max norm of $X$ are defined as $\| X \| = \sigma_1$, $\|X\|_{\star} = \sum_{i = 1}^{N \wedge p}\sigma_i$, $\| X\|_F = \sqrt{\sigma_1^2 + \dots + \sigma_{N \wedge p}^2}$ and $\| X\|_{\max} = \underset{ij}{\max}\ |X_{ij}|$ respectively.
The projection matrix for $X$ is defined as $P_X = X(X^\top X)^{-1}X^\top$.
We use $\operatorname{col}(X)$ to denote the column space of $X$ and $\sigma(z)$ to denote the sigmoid function $1/(1 + \exp(-z))$.

\section{Causal inference with low rank matrix factorization}
In this section, we first introduce the problem of causal inference under measurement noise and missing values formally and define notation.
We then show that the bias caused by measurement noise in linear regression is alleviated when more covariates are used.
Finally we review low rank matrix factorization methods and describe the proposed procedure for causal inference.

\subsection{Problem formulation}

We consider an observational study with $N$ subjects.
For subject $i$, $T_i$ is the treatment variable and we assume $T_i \in \{0, 1\}$ for simplicity.
We use $Y_i(0), Y_i(1)$ to denote the potential outcomes for subject $i$ under treatment and control respectively \citep{imbens2015causal}.
We can only observe the potential outcome corresponding to the treatment level that subject $i$ received, i.e., $Y_i = Y_i(T_i)$.
Assume that $\{Y_i(0), Y_i(1), T_i\}_{i=1}^N$ are independently and identically distributed (i.i.d). We denote $T = [T_1, ..., T_N]^\top$ and $Y = [Y_1, ..., Y_N]^\top$.
For the ease of exposition, we focus on estimating the average treatment effect (ATE):
\begin{equation*}
\tau = \mathbb{E}(Y_i(1) - Y_i(0)).
\end{equation*}

One standard way to estimate ATE is to adjust for the confounders. Suppose we have access to the confounders $U_i \in \mathbb{R}^r$ for subject $i$, $\forall i \in [N]$. Then we can employ many standard causal inference techniques (e.g., regression adjustment, propensity score reweighting, matching, etc.) to estimate ATE under the following unconfoundedness assumption:
\begin{assumption}[Unconfoundedness] \label{assumption: unconf}
For each $t = 0, 1$ and $i = 1, ..., N$, $Y_i(t)$ is independent of $T_i$ conditionally on $U_i$: $\prob(Y_i(t) \mid T_i, U_i) = \prob(Y_i(t) \mid U_i)$.
\end{assumption}

However, in practice we may not observe $\{U_i\}_{i=1}^N$ directly.
Instead suppose we can only partially observe covariates $X_i \in \mathbb{R}^p$, which is a collection of noisy measurements for the confounders.
The covariates $X_i$ can represent various data types by canonical encoding schemes.
For example, Boolean data is encoded using $1$ for true and $-1$ for false.
Many other encoding examples, e.g., categorical data or ordinal data, can be found in Udell et al. \citep{udell2016generalized}. We concatenate these covariates into $X \in \mathbb{R}^{N \times p}$.
We assume that only entries of $X$ over a subset of indices $\Omega \subset [N] \times [p]$ are observed and denote $\mathcal{P}_{\Omega}(X) = \sum_{(i, j) \in \Omega}X_{ij}e_ie_j^\top$ as the observed covariate matrix.

We further specify the generative model for individual entries $X_{ij}$, $(i, j) \in [N] \times [p]$.
We assume that $X_{ij}$ are drawn indepedently from distributions $\prob(X_{ij} \mid U_i^\top V_j)$, where $V_j \in \mathbb{R}^p$ represents loadings of the $j^{\text{th}}$ covariate on confounders.
The distribution $\prob(X_{ij} \mid U_i^\top V_j)$ models the \textit{measurement noise} mechanism for $X_{ij}$.
For example, if $X_{i1}$ is a measurement for $U_{i1}$ contaminated with standard Gaussian noise, then $\prob(X_{i1} \mid U_i^\top V_1) \sim \mathcal{N}(U_i^\top V_1, 1)$ where $V_1 = [1, 0, ..., 0]^\top $.
This generative model also accomodates \textit{proxy variables}.
Consider a simplified version of Spearman's measureable intelligence theory \citep{spearman1904general} where multiple kinds of test scores are used to characterize two kinds of (unobservable) intelligence: quantitative and verbal.
Suppose that there are $p$ tests (e.g., Classics, Math, Music, etc.) which are recorded in $X_{i1}, ..., X_{ip}$ and the two intelligence are represented by $U_{i1}$ and $U_{i2}$.
We assume that these proxy variables are noisy realizations of \textit{linear combinations} of two intelligence.
This can be modelled using the generative model $X_{ij} \sim \prob(X_{ij} \mid U_i^\top V_j)$ with $V_{j} = [V_{i1}, V_{i2}, 0, ..., 0]^\top $ for $j \in [p]$.
While this linear assumption seems restrictive, it's approximately true for a large class of nonlinear latent variable models when many proxies are used for a small number of latent variables \citep{udell2017nice}.

We aim to estimate ATE based on $\mathcal{P}_{\Omega}(X)$, $Y$ and $T$.
It is however very challenging for the presence of measurement noise and missing values.
One the one hand, most causal inference techniques cannot adapt to missing values directly and appropriate preprocessing is needed.
On the other hand, it is well known that measurement noise can dramatically undermine the unconfoundedness assumption and lead to biased causal effect estimation \citep{frost1979proxy, wickens1972note}, i.e., $\prob(Y_i(t) | T_i, X_{i}) \ne  \prob(Y_i(t) | X_{i})$ for $t = 0, 1$.

\subsection{Measurement noise and bias}
In this subsection, we show that using a large number of noisy covariates can effectively alleviate the ATE estimation bias resulted from measurement noise in linear regression setting.
Suppose there are no missing values, i.e., $\mathcal{P}_{\Omega}(X) = X$.
We consider the linear regression model: $\forall i \in [N]$, $Y_i = U_i^\top\alpha + \tau T_i + \epsilon_i$ , where $\alpha \in \mathbb{R}^r$ is the coefficient for confounders $U_i$, $\tau$ is the ATE, and $\epsilon_i$ are i.i.d sub-Gaussian error terms with mean $0$ and variance $\sigma^2$.
For $\forall i \in [N]$, $T_i$ are independently and probabilistically assigned according to confounders $U_i$. Unconfoundedness (Assumtpion \ref{assumption: unconf}) implies that $T_i$ are independent with $\epsilon_i$ conditionally on $U_i$.

\begin{proposition} \label{prop: bias}
Consider the additive noise model: $X = UV^\top + W$ where $\{U_i\}_{i=1}^N$ are i.i.d samples from a common distribution, $W \in \mathbb{R}^{N \times p}$ contains independent noisy entries with mean $0$ and variance $\sigma_w^2$, and entries in $W$ are independent with $\{U_i\}_{i=1}^N$. Suppose that $r$, $p$ are fixed and $p < N$.
As $N \to \infty$, the asymptotic bias of least squares estimator in linear regression of $Y_i$ on $X_i$ and $T_i$ has the following form:
\begin{align} \label{formula: bias}
 \frac{\mathbb{E}(T_iU_i)\mathbb{E}(U^\top_iU_i)^{-1}[\frac{1}{\sigma_w^2} V^\top V + \mathbb{E}(U^\top_iU_i)^{-1}]^{-1}\alpha}{\mathbb{E}(T_i^2) - \mathbb{E}(T_iU_i)[(\frac{1}{\sigma_w^2}V^\top V)^{-1} + \mathbb{E}(U^\top_iU_i)]^{-1}\mathbb{E}(U^\top_iT_i)}
\end{align}
\end{proposition}

\begin{corollary} \label{corollary: bias}
The asymptotic bias (\ref{formula: bias}) diminishes to $0$ when $\|V\| \to \infty$.
\end{corollary}

Corrolary \ref{corollary: bias} suggests an important fact: collecting a large number of noisy covariates is an effective remedy for the bias induced by measurement noise as long as the loadings of the covariates on latent confounders do not vanish too fast.
Surprisingly, in this independent additive noise case, the asymptotic bias (\ref{prop: bias}) is even nearly optimal: it is identical to the optimal asymptotic bias we would have if we knew the unobservable $V$ (Proposition 2, Appendix A).
In the rest of the paper, we further exploit this fact by using matrix factorization preprocessing which adapts to missing values, heterogenenous data types and more general noise models.

\subsection{Low rank matrix factorization preprocessing}


In this paper, we propose to recover the latent confounders $\{U_i \}_{i=1}^{N}$ from noisy and incomplete observations $\mathcal{P}_{\Omega}(X)$ by using low rank matrix factorization methods, which rely on the assumption:
\begin{assumption}[Low Rank Matrix] \label{assumption:low_rank}
The fully observed matrix $X$ is a noisy realization of a low rank matrix $\Phi \in \mathbb{R}^{N \times p}$ with rank $r \ll \min \{N, p\}$.
\end{assumption}

In the context of causal inference, Assumption \ref{assumption:low_rank} corresponds to the surrogate-rich setting where many proxies are used for a small number of latent confounders.
Under the generative model in section 2.1, Assumption \ref{assumption:low_rank} implies that $\Phi = UV^T$ where $U = [U_1, ..., U_N]^\top $ is the confounder matrix and $V = [V_1, ..., V_p]^T$ is the covariate loading matrix.
Although this assumption is unverifiable, low rank structure is shown to pervade in many domains such as images \citep{cao2015image}, customer preferences \citep{bennett2007netflix}, healthcare \citep{schuler2016discovering}, etc.
The recent work by Udell and Townsend \citep{udell2017nice} provides theoretical justifications that low rank structure arises naturally from a large class of latent variable models.

Moreover, low rank matrix factorization methods usually assume the \textit{Missing Completely at Random} (MCAR) setting where the observed entries are sampled uniformly at random \citep{gunasekar2014exponential,little2014statistical}.

\begin{assumption}[MCAR] \label{assumption: missing}
$\forall (i, j) \in \Omega$, $i \sim \operatorname{uniform}([N])$ and $j \sim \operatorname{uniform}([p])$ independently and the sampling is independent with the measurement noise.
\end{assumption}

Our paper takes the Exponential Family Matrix Completion (EFMC) as a concrete example, which further assumes exponential family noise mechanism \citep{gunasekar2014exponential}.

\begin{assumption} [Natural Exponential Family]
Suppose that each entry $X_{ij}$ is drawn independently from the corresponding \textit{natural exponential family} with $\Phi_{ij}$ as the natural parameter:
\begin{equation*}
\prob(X_{ij} | \Phi_{ij}) = h(X_{ij})\exp(X_{ij}\Phi_{ij} - G(\Phi_{ij}))
\end{equation*}
where $G: \mathbb{R} \to \mathbb{R}$ is a strictly convex and anlytic function called log-partition function. Furthermore, for some $\eta > 0$ and $\forall\ u \in \mathbb{R}$, $\nabla^2 G(u) \ge \operatorname{e}^{-\eta |u|}$.
\end{assumption}

Exponential family distributions encompass a wide variety of distributions like Gaussian, Poisson, Bernoulli that have been extensively used for modelling different data types \citep{mccullagh1984generalized}.
For example, if $X_{ij}$ takes binary values $\pm 1$, then we can model it using Bernoulli distribution: $\prob(X_{ij} \mid \Phi_{ij}) = \sigma(X_{ij}\Phi_{ij})$.
Moreover, it can be verified that the assumption on $\nabla^2 G(u)$ is satisfied by commonly used members of natural exponential family \citep{gunasekar2014exponential}.

EFMC estimates $\Phi$ by the following regularized \textit{M}-estimator:
\begin{equation} \label{formula: exp_family}\textstyle
\hat{\Phi} = \min_{\|\Phi\|_{\max} \le \frac{\alpha^*}{\sqrt{Np}}} \ \frac{Np}{|\Omega|}[\sum_{(i, j)\in \Omega} - \log \prob(X_{ij}| \Phi_{ij})] + \lambda \|\Phi\|_{\star}
\end{equation}
The estimator in (\ref{formula: exp_family}) involves solving a convex optimization problem, whose solution can be found efficiently by many off-the-shelf algorithms \citep{boyd2004convex}.
The nuclear norm regularization encourages a low-rank solution: the larger the tuning parameter $\lambda$, the smaller the rank of the solution $\hat{\Phi}$. In practice, $\lambda$ is usually selected by cross-validation.
Moreover, the constraint $\|\Phi\|_{\max} \le \frac{\alpha^*}{\sqrt{Np}}$ appears merely as an artifact of the proof and it is recommended to drop this constraint in practice \citep{kallus2016dynamic}.
It can be proved that under Assumptions $2-4$ and some regularity assumptions the relative reconstruction error of $\hat{\Phi}$ converges to $0$ with high probability (Lemma 4, Appendix A).
Furthermore, EFMC can be extended by using a rich library of loss functions and regularization functions \citep{udell2016generalized,singh2008unified}.

Suppose the solution $\hat{\Phi}$ from (\ref{formula: exp_family}) is of rank $\hat{r}$.
Then we can use its top $\hat{r}$ left singular matrix $\hat{U}$ to estimate the confounder matirx $U$.
The estimated confounder matrix $\hat{U}$ is used in place of the covariate matrix for subsequent causal inference methods (e.g., regression adjustment, propensity reweighting, matching, etc.). Admittedly, the confounder matrix $U$ can be identified only up to nonsingular linear transformation.
However, this suffices for many causal inference techniques. For example, regression adjustment methods based on linear regression \citep{wooldridge2015introductory}, polynomial regression, neural networks trained by backpropogation \citep{ng2004feature}, propensity reweighting or propensity matching using propensity score estimated by logistic regressions, and Mahalanobis matching are invariant to nonsingular linear transformations. Moreover, the invariance to linear transformation is important since the latent confounders may be abstract without commonly acknowledged scale (e.g., intelligence).

\section{Theoretical guarantee}
In this section, we theoretically justify matrix factorization preprocessing for estimating causal effect in linear regression setting.
We first identify the sufficient conditions on the estimated confounder matrix $\hat{U}$ for consistently estimating ATE in linear regression.
We then derive error bound for the induced ATE estimator with EFMC (\ref{formula: exp_family}) as the preprocessing step. Proofs are deferred to Appendix A.

Consider the linear regression model in Section 2.2. Suppose we use EFMC preprocessing and linear regression for causal inference, which leads to the ATE estimator  $\hat{\tau}$.
It is well known that the accuracy of $\hat{\tau}$ relies on how well the estimated column space $\operatorname{col}(\hat{U})$ approximates the column space of true confounder matrix $\operatorname{col}(U)$. Ideally, if $\operatorname{col}(\hat{U})$ aligns  with $\operatorname{col}(U)$ perfectly, then $\hat{\tau}$ is identical to the least squares estimator based on true confounders and is thus consistent.
We introduce the following distance metric between two column spaces \citep{cai2018rate}:

\begin{definition}
Consider two matrices $\hat{M} \in \mathbb{R}^{N \times k}$ and $M \in \mathbb{R}^{N \times r}$ with orthonormal columns, the principle angle between their column spaces is defined as
\begin{equation*}
\angle(M, \hat{M}) = \sqrt{1 - \sigma^2_{r \wedge k}(\hat{M}^\top M)}
\end{equation*}
\end{definition}
This metric measures the magnitude of the "angle" between two column spaces. For example, $\angle(M, \hat{M}) = 0$ if $\operatorname{col}(M) = \operatorname{col}(\hat{M})$ while $\angle(M, \hat{M}) = 1$ if they are orthogonal.

\begin{theorem} \label{theorem: linear-regression}
We assume the following assumptions hold:
(1) $\|\alpha\|_{\max} \le A$ for a positive constant $A$;
(2) $\frac{1}{\sqrt{Nr}}\|U\|$ is bounded above for any $N$;
(3) $\frac{1}{N}T^\top(I - P_U)T$ is bounded away from 0 for any $N$;
(4) $r\angle(\hat{U}, U) \to 0$ as $N \to 0$;
(5) Unconfoundedness (Assumption \ref{assumption: unconf}). Then
$\exists$ constant $c > 0$ such that with probability at least $1 - 2\exp(-cN^{1/2})$,
\begin{equation}
|\hat{\tau} - \tau^* | \le \frac{(\frac{2A}{\sqrt{N}}\|T\|)(\frac{1}{\sqrt{Nr}}\|U\|)({r}\angle(U, \hat{U})) - \frac{\sigma}{N^{1/4}}}{\frac{1}{N}T^\top(I - P_U)T  - \frac{2}{N}\|T\|^2\angle(U, \hat{U})} \overset{N \to \infty}{\longrightarrow} 0
\end{equation}
\end{theorem}
In the above theorem, assumption (3) is satisfied as long as the treatment variable is almost surely not a linear combination of the confounders (Lemma 7, Appendix A). Otherwise it is impossible to estimate ATE accurately due to multicollinearity.
Assumption (4) states that the column space of the estimated confounder matrix should converge to the true column space with rate faster than $1/r$ to guanrantee consistency of the resulting ATE estimator. This suggests that when the true rank $r$ grows with dimensions, estimating ATE consistently requires stronger column space convergence than merely estimating the true column space consistently, i.e., $\angle(U, \hat{U}) \to 0$.

Now we prove that EFMC leads to accurate ATE estimator with high probability under some generative assumptions on confounder matrix $U$ as well as covariate loading matrix $V$.

\begin{assumption} [Latent Confounders and Covariate Loadings] \label{assump: confound}
$U$ and $V$ satisfy the following for some positive constants $\underline{v}$, $\overline{v}$, $c_V$ and $c_L$:
(1) for $i \in [N]$, $U_i$ are i.i.d Gaussian samples with covariance matrix $\Sigma_{r \times r} = LL^\top$ for some full rank matrix $L \in \mathbb{R}^{r \times r}$ such that $\frac{1}{\sqrt{r}}\|L\| < c_L$;
(2) $\underline{v} p \le \sigma_r^2(VL^\top) \le \sigma_1^2(VL^\top) \le \overline{v} p$ and $\frac{\max_j\|V_j\|}{\|V\|_F} \le \frac{c_V}{\sqrt{p}}$, $j = 1, ..., p$.
\end{assumption}

Assumption \ref{assump: confound} specifies a Gaussian random design for latent confounders, which implies assumption (2) in Theorem \ref{theorem: linear-regression} with high probability (Lemma 8, Appendix).
It also assumes without loss of generality that the latent confounders are not perfectly linearly correlated.
Moreover, Assumption \ref{assump: confound} exludes the degenerate case where almost all covariates have vanishing loadings on the latent confounders, i.e., $\frac{\max_j\|V_j\|}{\|V\|_F} \approx \frac{\max_j\|V_j\|}{\sqrt{n_V} \max_j\|V_j\|} = \frac{1}{\sqrt{n_V}}$ where $n_V$ is the numebr of covariates with nonvanishing loadings and $n_V$ scales much slower than $p$. In this case, the collected covariates are not informative enough for recoverying the latent confounders.

\begin{theorem} \label{theorem: exp}
Let $X_{ij}$ be sub-Exponential conditionally on $U_i$ with parameter $\sigma'$ for $\forall (i, j)$ and $T_i$ is almost surely not a linear combination of $U_i$. Suppose EFMC is used as the preprocessing step with $\lambda = 2c_0\sigma'\sqrt{Np}\sqrt{\frac{r\overline{N}\log\overline{N}}{|\Omega|}}$, where $\overline{N} = N \vee p$ and $|\Omega| > c_1r\overline{N}\log\overline{N}$ for positive constants $c_0$ and $c_1$ . Assume $r/N \to 0$ and $\exists \delta > 0$ such that $p^{1 + \delta}/N \to 0$. Under Assumption $1 - 5$, assumptions (2)-(4) in Theorem \ref{theorem: linear-regression} hold with high probability. Furthermore, $\exists$ positive constants $c_2$, $c_3$, $c_{\sigma', \eta}$ such that, the following holds with probability at least $1 - c_2\operatorname{exp}(-c_3N^{1/2}) - c_2N^{-1/2} - 2\operatorname{exp}(-c_3p^{\delta})$,
\begin{equation}\label{formula: error_bound}
|\hat{\tau} - \tau| \le \frac{Ac_Lc_{\sigma', \eta}c_V\sqrt{\frac{r^5\overline{r}\overline{N}\log\overline{N}}{|\Omega|}} -  \frac{\sigma}{N^{1/4}}[\sqrt{\frac{\underline{v}}{\underline{v} + 2\overline{v}}} - \Lambda(r, \overline{N}, |\Omega|)]}{[\sqrt{\frac{\underline{v}}{\underline{v} + 2\overline{v}}} - \Lambda(r, \overline{N}, |\Omega|)][\frac{1}{N}T^\top(I - P_U)T - 2\Lambda(r, \overline{N}, |\Omega|)]}
\end{equation}
where $\Lambda(r, \overline{N}, |\Omega|) = c_{\sigma', \eta}c_V\sqrt{\frac{\bar{r}r^3\overline{N}\log\overline{N}}{|\Omega|}}$ and $\overline{r} = \max\{r, \log\overline{N}\}$.
\end{theorem}

The assumption that $X_{ij}$ is sub-Exponential encompasses common exponential family distributions like Gaussian, Bernoulli, Poisson, Binomial, etc.
The assumption that $p^{1 + \delta}/N \to 0$ appears as an artifact of proof and our simulation shows that the consistency also holds when $N < p$ (Figure $3$, Appendix B).
Theorem \ref{theorem: exp} guarantees that the ATE estimator induced by EFMC is consistent as long as $r^5\overline{r}\overline{N}\log\overline{N}/|\Omega| \to 0$ when $N, p \to \infty$. This seems much more restrictive than consistent matrix reconstruction that merely requires $r\overline{N}\log\overline{N}/|\Omega| \to 0$ (Lemma 4, Appendix A). However, this is due to the pessimistic nature of the error bound. Our simulations in Section 4.1 show that matrix factorization works very well for $r = 5$, $N = 1500$ and $p = 1450$ such that $r^6 \gg N$.

\section{Numerical results}

In this part, we illustrate the effectiveness of low rank matrix factorization in alleviating the ATE estimation error caused by measurement noise using synthetic datasets with both continuous and binary covariates and the twins dataset introduced by Louizos et al. \citep{louizos2017causal}. For the implementation of matrix factorization, we use the following nonconvex formulation:
\begin{equation} \label{formula: glrm}\textstyle
\hat{U}, \hat{V} = \underset{{U \in \mathbb{R}^{N \times k}, V \in \mathbb{R}^{p \times k}}}{\operatorname{argmin}} \sum_{(i, j) \in \Omega} L_{i, j}(X_{ij}, U_{i}^\top V_j) +  \frac{\lambda}{2}(\| U\|_F + \| V \|_F)
\end{equation}
where $L_{ij}$ is a loss function assessing how well $U_{i}^\top V_j$ fits the observation $X_{ij}$ for $(i, j) \in \Omega$. The solution $\hat{U}$ can be viewed as the estimated confounder matrix.
This nonconvex formulation (\ref{formula: glrm}) is proved to equivalently recover the solution of the convex formulation (\ref{formula: exp_family}) when log-likelihood loss functions and sufficient large $k$ are used \citep{udell2016generalized,kallus2016dynamic}.
Solving the nonconvex formulation (\ref{formula: glrm}) approximately is usually much faster than solving the convex counterpart.
In our experiments, we use the the R package softImpute \citep{hastie2015matrix} when dealing with continuous covariates and quadratic loss, the R package logisticPCA \citep{collins2002generalization} when dealing with binary covariates and logistic loss, and the Julia package LowRankModel \citep{udell2016generalized} when dealing with categorical variables and multinomial loss.
 All tuning parameters are chosen via $5$-fold cross-validation.

\subsection{Synthetic experiment}
We generate synthetic samples according to the following linear regression process:
$Y_i \mid U_i, T_i \sim \mathcal{N}(\alpha^\top U_i + \tau T_i, 1)$ where confounder $U_{ij} \sim \mathcal{N}(0, 1)$ and treatment variable $T_i \mid U_i \sim \operatorname{Bernoulli}(\sigma(\beta^\top U_i))$ for $i \in [N]$, $j \in [r]$.
We consider covariates generated from both indepedent Gaussian noise and independent Bernoulli noise: $X_{ij} \sim \mathcal{N}(U_i^\top V_j, 5)$ and $X_{ij} \sim \operatorname{Bernoulli}(\sigma(U_i^\top V_j))$ for $V_j \in \mathbb{R}^r$.
We set the dimension of the latent confounders $r = 5$, use $\alpha = [-2, 3, -2, -3, -2]$ and $\beta = [1, 2, 2, 2, 2]$, and choose $\tau = 2$ in our example. (But our conclusion is robust to different values of these parameter.)
We consider low dimensional case where the number of covariates $p$ varies from $100$ to $1000$ and the sample size $N = 2p$ and high dimensional case where $p$ varies from $150$ to $1500$ and $N = p + 50$.
For each dimensional setting, we compute the error metrics based on $50$ replications of the experiments and we generate entries of $V$ independently from standard normal distribution with $V$ fixed across the replications.

\begin{figure} \label{figure: synthetic}
  \centering
    \includegraphics[width = \linewidth]{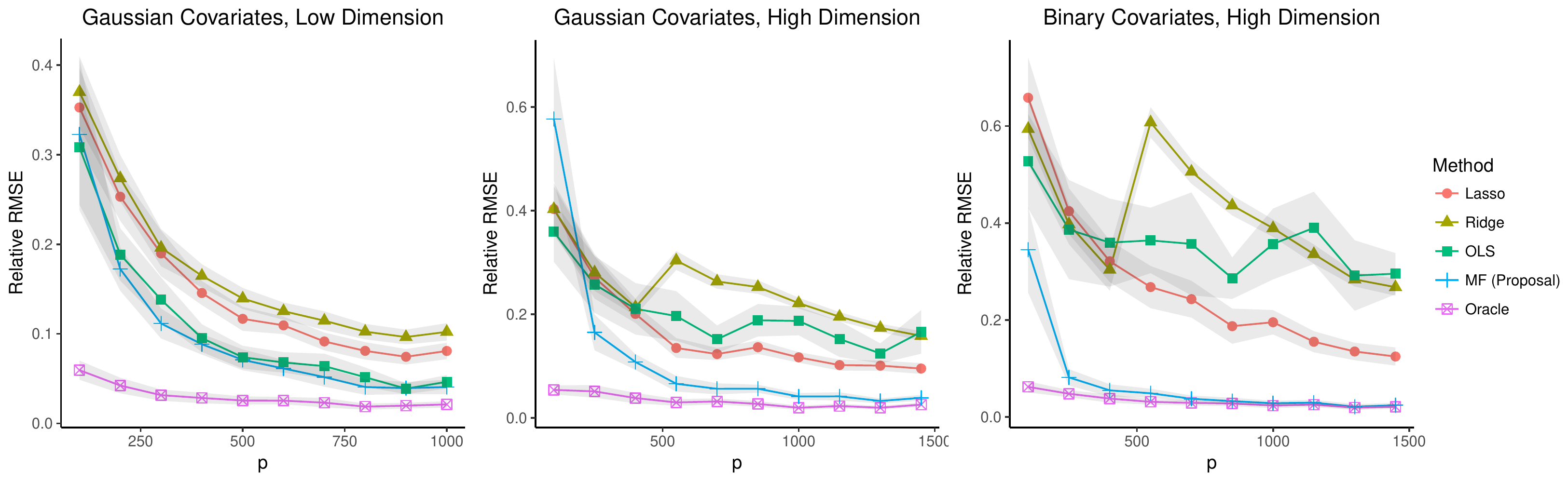}
  \caption{Results from experiments on synthetic data. }
\end{figure}

We compare the root mean squared error (RMSE) scaled by the true ATE in Figure \ref{figure: synthetic} for the following five ATE estimators in linear regression:
the Lasso, Ridge and OLS estimators from regressing $Y_i$ on $T_i$ and noisy covariates $X_i$, the OLS estimator from regressing $Y_i$ on $T_i$ and the estimated confounders $\hat{U}_i$ from matrix factorization (MF), and the OLS estimator from regressing $Y_i$ on $T_i$ and the true confounders $U_i$ (Oracle). The shaded area corresponds to the $2$-standard-deviation error band for the estimated relative RMSE across $50$ replications.

Figure $1$ shows that OLS leads to accurate ATE estimation for Gaussian additive noise when the number of covariates is sufficiently large, which is consistent with Corollary \ref{corollary: bias}.
However, for high dimensional data, matrix factorization preprocessing dominates all other feasible methods and its RMSE is very close to the oracle regression for sufficiently large number of covariates.
While all feasible methods tend to have better performance when more covariates are available, matrix factorization preprocessing is the most effective in exploiting the noisy covariates for accurate causal inference.
Sufficiently many noisy covariates are very important for accurate ATE estimation in the presence of measurement noise. We can show that the error does not converge when only $N$ grows but $p$ is fixed (Figure $5$, Appendix B).
With only a few covariates, matrix factorization preprocessing may have high error because the cross-validation chooses rank smaller than the ground truth.
Furthermore, the gain from using matrix factorization is more dramatic for binary covariates, which demonstrates the advantage of matrix factorization preprocessing with loss functions adapting to the data types.
More numerical results on different dimensional settings and missing data can be found in Appendix.

\subsection{Twin mortality}
We further examine the effectiveness of matrix factorization preprocessing using the twins dataset introduced by Louizos et al. \citep{louizos2017causal}.
This dataset includes information for $N = 11984$ pairs of twins of same sex who were born in the USA between 1998-1991 and weighted less than $2$kg.
For the $i^{\text{th}}$ twin-pair, the treatment variable $T_i$ corresponds to being the heavier twin and the outcomes $Y_i(0), Y_i(1)$ are the mortality in the first year after they were born.
We have outcome records for both twins and view them as two potential outcomes for the treatment variable. Therefore, the $-2.5\%$ difference between the average mortality rate of heavier twins and that of ligher twins can be viewed as the "true" ATE.
This dataset also includes other $46$ covariates relating to the parents, the pregnancy and birth for each pair of twins. More details about the dataset can be found in Louizos et al. \citep{louizos2017causal}.

\begin{figure}\label{figure: twins}
  \centering
  \includegraphics[width = \linewidth]{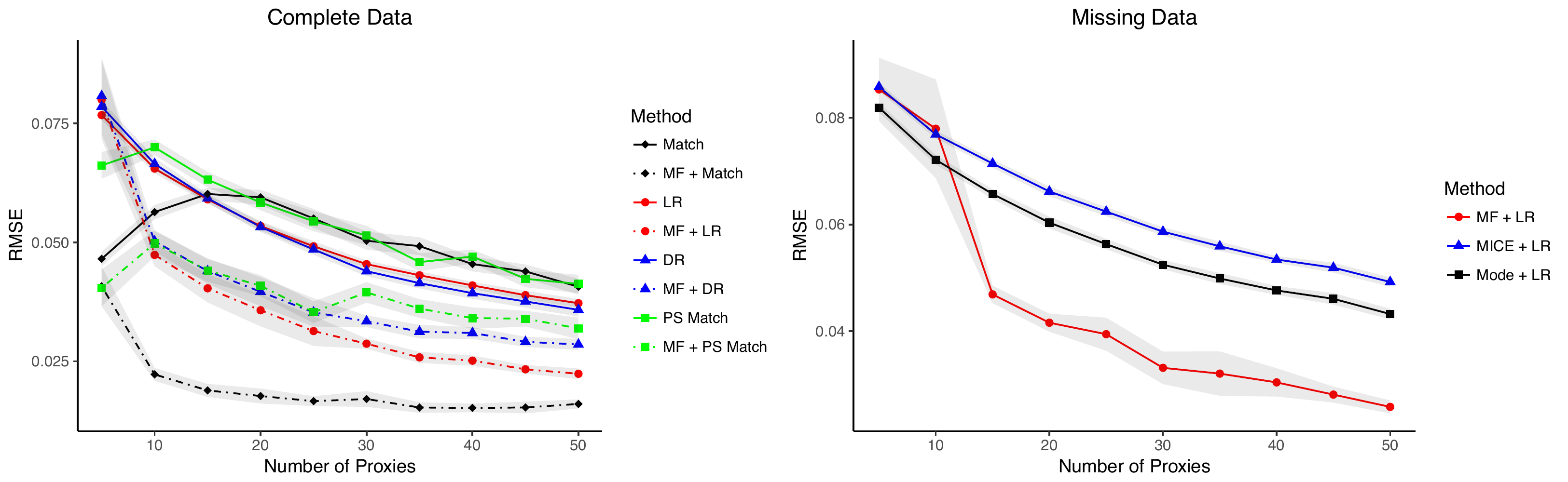}
  \caption{Results on the twins dataset.}
\end{figure}

To simulate confounders in observational studies, we follow the practice in Louizos et al. \citep{louizos2017causal} and selectively hide one of the two twins based on one variable highly correlated with the outcome: GESTAT10, the number of gestation weeks prior to the birth.
This is an ordinal variable with values from $0$ to $9$ indicating less than $20$ gestation weeks, $20 - 27$ gestation weeks and so on. We simulate $T_i \mid U_i \sim \operatorname{Bernoulli}(\sigma(5(U_i/10 - 0.1)))$, where $U_i$ is the confounder GESTAT10.
Then for each twin-pair, we only observe the lighter twin if $T_i = 0$ and the heavier twin otherwise.
We create noisy proxies for the confounder as follows: we replicate the GESTAT10 $p$ times and independently perturb the entries of these $p$ copies with probability $0.5$.
Each perturbed entry is assigned with a new value sampled from $0$ to $9$ uniformly at random.
We denote these proxy variables as $\{X_i\}_{i=1}^N$.
We also consider the presence of missing values: we set each entry as missing value independently with probability $0.3$.
We vary $p$ from $5$ to $50$ and for each $p$ we repeat the experiments $20$ times for computing error metrics.

We compare the performance of different methods for both complete data and missing data in Figure \ref{figure: twins}.
For complete data, we consider logistic regression (LR), doubly robust estimator (DR), Mahalanobis matching (Match) and propensity score matching (PS Match) using $\{X_i\}_{i=1}^N$, and their counterparts using the estimated confounders $\{\hat{U}_i\}_{i=1}^N$ from matrix factorization. All propensity scores are estimated by logistic regression using $\{X_i\}_{i=1}^N$ or $\{\hat{U}_i\}_{i=1}^N$ accordingly. The matching methods are implemented via the full match algorithm in the R package optmatch \citep{fullmatch}.
For missing data, we consider logistic regression using data output from different preprocessing method: imputing missing values by column-wise mode, multiple imputation using the R package MICE  with $5$ repeated imputations \citep{mice}, and the estimated confounders $\{\hat{U}_i\}_{i=1}^N$ from matrix factorization.

We can observe that all methods that use matrix factorization clearly outperform their counterparts that do not, even though the noise mechanism does not obey common noise assumptions in matrix factorization literature. This also demonstrates that matrix factorization preprocessing can augment popular causal inference methods beyond linear regression. Furthermore, matrix factorization preprocessing is robust to a considerable amount of missing values and it dominates both the ad-hoc mode imputation method and the state-of-art multiple imputation method. This suggests that inferring the latent confounders is more important for causal inference than imputing the noisy covariates.

\section{Conclusion}
In this paper, we address the problem of measurement noise prevalent in causal inference.
We show that with a large number of noisy proxies, we can reduce the bias
resulting from measurement noise by using matrix factorization preprocessing to infer latent confounders.
We guarantee the effectiveness of this approach in a linear regression setting, and show its effectiveness numerically on both synthetic and real clinical datasets.
These results demonstrate that preprocessing by matrix factorization to infer latent confounders
has a number of advantages: it can accommodate a wide variety of data types,
ensures robustness to missing values, and
can improve causal effect estimation when used in conjunction
with a wide variety of causal inference methods.
As such, matrix factorization allows more principled and accurate
estimation of causal effects from observational data.

\bibliographystyle{unsrt}
\medskip
\bibliography{paper}

\newpage

\appendix
\section{Proofs}
\subsection{Measurement Noise and Bias}

\begin{proof}[Proof of Proposition 1]
\begin{equation}\label{formula: ols}
\hat{\tau} =  [\frac{1}{N}T^\top(I - P_{X})T]^{-1}[\frac{1}{N}T^\top(I - P_{X})Y]
\end{equation} 
By Law of Large Number,
\begin{equation*}
 \frac{1}{n}T^\top X \to \expect[T_i (U_iV^\top + W_i)] = \expect(T_iU_i)V^\top
\end{equation*}
\begin{equation*}
 \frac{1}{N}X^\top Y \to \expect[(U_iV^\top + W_i)^\top (U_i\alpha + \tau T_i + \epsilon_i)] = V \expect[U^\top_i U_i]\alpha + \tau V \expect(U^\top_iT_i)
\end{equation*}
\begin{equation*}
\frac{1}{N} T^\top Y \to \expect[T_i(U_i \alpha + \tau T_i + \epsilon_i)] = \expect(T_iU_i)\alpha + \tau \expect(T_i^2)
\end{equation*}
\begin{equation*}
 (\frac{1}{N}X^\top X)^{-1} \to [\expect(V U^\top_iU_iV^\top + {W^\top_iU_iV^\top} + {V U^\top_iW_i} + {W^\top_iW_i})]^{-1} = [V\expect(U^\top_iU_i)V^\top + \sigma^2_{w}I_{r \times r}]^{-1} 
\end{equation*}
By Sherman–Morrison–Woodbury formula, 
\begin{align*}
[V\expect(U^\top_iU_i)V^\top + \sigma^2_{w}I_{r \times r}]^{-1} &= \frac{1}{\sigma_w^2}I - \frac{1}{\sigma_w^2}V[(\frac{1}{\sigma^2_w}\expect U^\top_iU_i)^{-1} + V^\top V ]^{-1}V^\top \\
\end{align*}
and 
\begin{equation*}
[(\frac{1}{\sigma^2_w}\expect U^\top_iU_i)^{-1} + V^\top V ]^{-1} = (V^\top V)^{-1} - (V^\top V)^{-1}[(V^\top V)^{-1} + \frac{1}{\sigma^2_w}\expect U^\top_iU_i]^{-1}(V^\top V)^{-1}
\end{equation*}
Plug these terms back in \ref{formula: ols}, 
\begin{align*}
\frac{1}{N} T^\top(I - P_X)Y &= \tau \expect(T_i^2) + \expect(T_iU_i)\alpha \\
&+  \expect(T_iU_i)V^\top[V(\expect U^\top_iU_i)^{-1}V^\top + \sigma^2_wI]^{-1}V \expect U^\top_iU_i\alpha \\
&+  \tau \expect(T_iU_i)V^\top[V(\expect U^\top_iU_i)^{-1}V^\top + \sigma^2_wI]^{-1}V^top \expect U^\top_iT_i \\
&= \tau \expect(T_i^2) + \expect(T_iU_i)\alpha \\
&+  \frac{1}{\sigma^2_w}\expect(T_iU_i)\{V^\top V[(\frac{1}{\sigma^2_w}\expect U^\top _iU_i)^{-1} + V^\top V]^{-1}V^\top V - V^\top V\}\expect(U^\top_iU_i)\alpha \\
&+ \frac{\tau}{\sigma^2_w}\expect(T_iU_i)\{V^\top V[(\frac{1}{\sigma^2_w}\expect U^\top _iU_i)^{-1} + V^\top V]^{-1}V^\top V - V^\top V\}\expect(U^\top_iT_i) \\
&= \tau \expect(T_i^2) + \expect(T_iU_i)\alpha - \frac{1}{\sigma^2_w}\expect(T_iU_i)\{\frac{1}{\sigma^2_w}\expect(U^\top_iU_i) + (V^\top V)^{-1}\}^{-1}E(U^\top_iU_i)\alpha \\
&- \frac{\tau}{\sigma^2_w}E(T_iU_i)\{\frac{1}{\sigma^2_w}E(U^\top_iU_i) + (V^\top V)^{-1}\}^{-1}E(U^\top_iT_i)
\end{align*}
Similarly,
\begin{align*}
\frac{1}{N} T^\top(I - P_X)T  &\to \expect(T_i^2) - \expect(T_iX_i)(\expect X^\top_iX_i)^{-1}\expect(X^\top_i T_i) \\
&= \expect(T_i^2) - \frac{1}{\sigma^2_w}\expect(T_iU_i)\{\frac{1}{\sigma^2_w}\expect(U^\top_iU_i) + (V^\top V)^{-1}\}^{-1}\expect(U^\top_iT_i) 
\end{align*}
Therefore, 
\begin{align*}
\hat{\tau} - \tau &\to \frac{ \expect(T_iU_i)\alpha - \frac{1}{\sigma^2_w}\expect(T_iU_i)\{\frac{1}{\sigma^2_w}\expect(U^\top_iU_i) + (V^\top V)^{-1}\}^{-1}E(U^\top_iU_i)\alpha}{\expect(T_i^2) - \frac{1}{\sigma^2_w}\expect(T_iU_i)\{\frac{1}{\sigma^2_w}\expect(U^\top_iU_i) + (V^\top V)^{-1}\}^{-1}\expect(U^\top_iT_i) } \\
&= \frac{\mathbb{E}(T_iU_i)\mathbb{E}(U^\top_iU_i)^{-1}[\frac{1}{\sigma_w^2} V^\top V + \mathbb{E}(U^\top_iU_i)^{-1}]^{-1}\alpha}{\mathbb{E}(T_i^2) - \mathbb{E}(T_iU_i)[(\frac{1}{\sigma_w^2}V^\top V)^{-1} + \mathbb{E}(U^\top_iU_i)]^{-1}\mathbb{E}(U^\top_iT_i)}
\end{align*}
The last equality once again follows from Sherman–Morrison–Woodbury formula.
\end{proof}

\begin{proof}[Proof for Corollary \ref{corollary: bias}]
$\|V^\top V\| = \|V\|^2  \to \infty$ so $\|[\frac{1}{\sigma_w^2} V^\top V + \mathbb{E}(U^\top_iU_i)^{-1}]^{-1}\| \to 0$. On the other hand, by Sherman–Morrison–Woodbury formula,  
\begin{align*}
& \quad \quad \quad \quad \quad \quad \quad \quad \quad \quad \mathbb{E}(T_iU_i)[(\frac{1}{\sigma_w^2}V^\top V)^{-1}\mathbb{E}(U^\top_iU_i)]^{-1}\mathbb{E}(U^\top_iT_i) \\
&= \mathbb{E}(T_iU_i)\mathbb{E}(U^\top_iU_i)^{-1}\mathbb{E}(U^\top_iT_i) - \mathbb{E}(T_iU_i)\mathbb{E}(U^\top_iU_i)^{-1}[\mathbb{E}(U^\top_iU_i) + \frac{1}{\sigma_w^2}V^\top V]^{-1}\mathbb{E}(U^\top_iU_i)^{-1}\mathbb{E}(U^\top_iT_i)
\end{align*}
So the denominator term satisfies that 
\begin{equation*}
\mathbb{E}(T_i^2) -  \mathbb{E}(T_iU_i)[(\frac{1}{\sigma_w^2}V^\top V)^{-1} + \mathbb{E}(U^\top_iU_i)]^{-1}\mathbb{E}(U^\top_iT_i) > \mathbb{E}(T_i^2) - \mathbb{E}(T_iU_i)[\mathbb{E}(U^\top_iU_i)]^{-1}\mathbb{E}(U^\top_iT_i)
\end{equation*}
which is bounded away from $0$ by Lemma 7. Therefore, the asymptotic bias term (\ref{formula: bias}) diminishes to $0$. 
\end{proof}

\begin{proposition}\label{prop: equivalence}
Given the true $V$, $U_i$ can be estimated by the OLS estimator for the following linear regression: for $j = 1, \dots, p$
\begin{equation*}
X_{ij} = V_jU_i^\top + \eta_{ij}.
\end{equation*}
Namely we regress $X^\top_i$ on the design matrix $V$ to estimate $U^\top_i$. The resulting confounder estimator is $\tilde{U} = XV(V^\top V)^{-1}$. The subsequent OLS estimator for ATE based on $\tilde{U}$, $Y$ and $T$ is denoted as $\hat{\tau}$. Under the assumptions in Proposition 1,  $\tilde{\tau}$ has the same asymptotic bias as in Proposition 1.
\end{proposition}
\begin{proof}[Proof of Proposition 2]
In this case, $\tilde{U} = U + WV(V^\top V)^{-1}$. Moreover, $\tilde{\tau}$ has the following form:
\begin{align*}
\hat{\tau} &= [\frac{1}{N}T^\top(I - P_{\tilde{U}})T]^{-1}[\frac{1}{N}T^\top(I - P_{\tilde{U}})Y] \\
&= [\frac{1}{N}T^\top(I - P_{\tilde{U}})T]^{-1}[\frac{1}{N}T^\top(I - P_{\tilde{U}})U\alpha + \frac{\tau}{N}T^\top(I - P_{\tilde{U}})T + \frac{1}{N}T^\top(I - P_{\tilde{U}})\epsilon]
\end{align*}
Take $\frac{1}{N}T^\top P_{\tilde{U}}U\alpha$ as an example. 
\begin{equation*}
 \frac{1}{N}T^\top P_{\tilde{U}}U\alpha =  \frac{1}{N}T\tilde{U}(\frac{1}{N}\tilde{U}^\top \tilde{U})^{-1}\frac{1}{N}\tilde{U}^\top U\alpha
\end{equation*}
where
\begin{equation*}
\frac{1}{N}T\tilde{U} = \frac{1}{N}T^\top [U + WV(V^\top V)^{-1}] \to \expect T_iU_i
\end{equation*}
\begin{align*}
(\frac{1}{N}\tilde{U}^\top \tilde{U})^{-1} &= \{\frac{1}{N}[U + WV(V^\top V)^{-1}]^\top[U + WV(V^\top V)^{-1}]\}^{-1} \\
&\to \frac{1}{\sigma_w^2} [\frac{1}{\sigma_w^2} \expect U_i^\top U_i + (V^\top V)^{-1}]^{-1}
\end{align*}
\begin{equation*}
\frac{1}{N}\tilde{U}^\top U\alpha = \frac{1}{N}[U + WV(V^\top V)^{-1}]^\top U\alpha \to \expect U_i^\top U_i\alpha
\end{equation*}
Therefore
\begin{align*}
 \frac{1}{N}T^\top P_{\tilde{U}}U\alpha &\to \frac{1}{\sigma_w^2} \expect T_iU_i[\frac{1}{\sigma_w^2} \expect U_i^\top U_i + (V^\top V)^{-1}]^{-1}\expect U_i^\top U_i\alpha
\end{align*}
which is exactly equal to the limit of $\frac{1}{N}T^\top P_{X}U\alpha$ in the proof of Proposition 1. The equivalence of other terms can be verified similarly.
\end{proof}

\subsection{Proof of Theorem \ref{theorem: linear-regression}}
\begin{proof}[Proof of Theorem \ref{theorem: linear-regression}]
The error of the ATE estimator in the linear regression can be written as:
\begin{equation} \label{formula: error}
\hat{\tau} - \tau = [\frac{1}{N}T^\top(I - P_{\hat{U}})T]^{-1}[\frac{1}{N}T^\top(I - P_{\hat{U}})U]\alpha + [\frac{1}{N}T^\top(I - P_{\hat{U}})T]^{-1}[\frac{1}{N}T^\top(I - P_{\hat{U}})\epsilon]
\end{equation} 

We first bound $\frac{1}{N}[T^\top(I - P_{\hat{U}})U]\alpha$:
\begin{align*}
\frac{1}{N}|T^\top(I - P_{\hat{U}})U\alpha| &= \frac{1}{N}|T^\top(I - P_{\hat{U}})U\alpha - T^\top(I - P_{{U}})U\alpha| \\
&= \frac{1}{N}|T^\top(P_U - P_{\hat{U}})U\alpha| \\
&\le \frac{1}{\sqrt{N}}\|T\|\frac{1}{\sqrt{N}}\|U\alpha\|\|P_{U} - P_{\hat{U}}\| \\
&\le (\frac{2}{\sqrt{N}}\|T\|)(\frac{A}{\sqrt{Nr}}\|U\|)(r\angle(\hat{U}, U))
\end{align*}
The first equaility follows from $(I - P_U)U = 0$. The last inequality follows from Lemma \ref{lemma: space_dist}.

We then bound $[T^\top(I - P_{\hat{U}})T]$:
\begin{align*}
\frac{1}{N}|T^\top(I - P_{\hat{U}})T| &= \frac{1}{N}|T^\top(I - P_{U})T + T^\top(P_U - P_{\hat{U}})T| \\
&\ge \frac{1}{N}T^\top(I - P_{U})T - |\frac{1}{N}T^\top(P_U - P_{\hat{U}})T| \\
&\ge \frac{1}{N}T^\top(I - P_{U})T - \frac{2}{N}\|T\|^2\angle\Theta(U, \hat{U}) \\
\end{align*}

Furthermore, we can bound $\frac{1}{N}|T^\top(I - P_{\hat{U}})\epsilon|$: $T^\top(I - P_{\hat{U}})\epsilon$ is sub-Gaussian with mean $0$ and variance $\sigma^2T^\top(I - P_{\hat{U}})T$. By Hoeffding bound, for any $t > 0$ and some constant $c >0$,
\begin{equation*}
P(\frac{1}{n}|T^\top(I - P_{\hat{U}})\epsilon| \ge t) \le 2e^{-\frac{cN^2t^2}{\sigma^2T^\top(I - P_{\hat{U}})T}} \le 2e^{-\frac{cNt^2}{\sigma^2}} 
\end{equation*}
Take $t = \frac{\sigma}{N^{1/4}}$, then $\frac{1}{N}|T'(I - P_{\hat{U}})\epsilon| \le \frac{\sigma}{N^{1/4}}$ with high probability $1 - 2\exp(-cN^{1/2})$ for some positive constant $c$. 

Plug these three bounds in (\ref{formula: error}) then the conclusion follows.
\end{proof}

\begin{lemma}[Equivalence of Space Distance Metrics] \label{lemma: space_dist}
The metric $\angle(\hat{M}, M)$ for matrices $M \in \mathbb{R}^{N \times r}$ and $\hat{M} \in \mathbb{R}^{N \times k}$ with orthonormal columns satisfies:
\begin{equation*}
\angle(\hat{M}, M) \le \|\hat{M}\hat{M}^T - MM^T\| \le 2\angle(\hat{M}, M) 
\end{equation*}
\end{lemma}
\begin{proof}
See Lemma 1 in Cai et al. \citep{cai2018rate}.
\end{proof}




\subsection{Proof of Theorem \ref{theorem: exp}}
\begin{proof}[Proof of Theorem \ref{theorem: exp}]
Lemma $5-10$ show that the Assumption $(2)-(4)$ in Theorem \ref{theorem: linear-regression} hold with high probability. The conclusion follows by plugging in Theorem \ref{theorem: linear-regression} the bounds in Lemma \ref{lemma: exp_column_space}, \ref{lemma: bound_U}, \ref{lemma: spikeness}, \ref{lemma: sv_ratio}.
\end{proof}

\begin{lemma} \label{lemma: exp_family}
Assume that $\Phi_{N \times p}$ is a low-rank matrix of rank atmost $r \ll N, p$. Further assume $\forall (i, j)$, $X_{ij} - g(\Phi_{ij})$ are sub-exponential with parameter $\sigma'$ and $|\Omega| > c_0r\bar{N}log\bar{N}$ for large enough constant $c_0$. Given any $\beta$ there exist positive constants $c_{\beta}, C_{\beta}, K_{\beta}$ such that for $\lambda = 2c_{\beta}\sigma'\sqrt{Np}\sqrt{\frac{r\bar{N}\log\bar{N}}{|\Omega|}}$, the estimator from Exponential Family Matrix Completion (\ref{formula: exp_family}) satisfies the following with probability at least $1 - 4e^{-(1+\beta)\log^2\bar{N}} - e^{-(1+\beta)\log \bar{N}}$:
\begin{equation}
\|\hat{\Phi} - \Phi\|_F^2 \le C_{\beta}\frac{\alpha_{sp}^2(\Phi)\max\{\sigma'^2, 1\}}{\mu_{\beta}^2}(\frac{r\bar{N}log\bar{N}}{|\Omega|})\|\Phi\|_F^2
\end{equation}
where $\mu_{\beta} = K_{\beta}e^{-\frac{2\eta\alpha_{sp}(\Phi)}{\sqrt{Np}}} > 0$ for some positive constant $K_{\beta}$ and $\alpha_{sp}(\Phi)$ is the spikeness ratio of $\Phi$ defined as follows:
\begin{equation*}
\alpha_{sp}(\Phi) = \frac{\|\Phi\|_{\max}\sqrt{Np}}{\|\Phi\|_F }
\end{equation*}
\end{lemma}
\begin{proof}
See Corollory 1 in Gunasekar et al. \citep{gunasekar2014exponential} for sub-Gaussian $X_{ij} - g(\Phi_{ij})$. For sub-Exponential case, use the Orlicz norm corresponding to sub-Exponential random variables for Lemma 3 and Lemma 5 in Gunasekar et al. \citep{gunasekar2014exponential} and then the same conclusion follows. 
\end{proof}

\begin{lemma} \label{lemma: exp_column_space}
Suppose the resulting estimator from Exponential Family Matrix Completion (\ref{formula: exp_family}) has singular value decomposition $\hat{\Theta} = \hat{U}\hat{\Sigma}\hat{V}^\top$. Under the assumptions in Lemma \ref{lemma: exp_family}, there exists positive constant $c_{\sigma', \eta}$ such that the following holds with probability at least $1 - 4e^{-2\log^2\bar{N}} - e^{-2\log \bar{N}}$:
\begin{equation*}
\angle sin\Theta(\hat{U}, U) \le \frac{c_{\sigma', \eta}\alpha_{sp}(\Phi)\sqrt\frac{r^3\bar{N}\log\bar{N}}{|\Omega|}}{\frac{\sigma_r(\Phi)}{\sigma_1(\Phi)} - c_{\sigma', \eta}\alpha_{sp}(\Phi)\sqrt\frac{r^3\bar{N}\log\bar{N}}{|\Omega|}}\land 1 
\end{equation*} 
\end{lemma}
\begin{proof}
Obviously $\frac{\alpha_{sp}(\Phi)}{\sqrt{Np}} < 1$, so $\mu_{\beta} > K_{\beta}e^{-2\eta}$. Let $c^2_{\sigma', \eta} = \frac{C_{\beta}\max\{\sigma'^2, 1\}}{K^2_{\beta}e^{-4\eta}}$ with $\beta = 1$, then according to Lemma \ref{lemma: exp_family}, the following holds with high probability at least $1 - 4e^{-2\log^2\bar{N}} - e^{-2\log \bar{N}}$:
\begin{equation*}
\|\hat{\Phi} - \Phi\|_F^2 \le c^2_{\sigma', \eta}\alpha^2_{sp}(\Phi)\frac{r\bar{N}log\bar{N}}{|\Omega|}\|\Phi\|_F^2
\end{equation*}
Apply Wedin's Theorem (Theorem \ref{theorem: wedin}) on $\Phi$ and $\hat{\Phi}$ with $E = \Phi - \hat{\Phi}$. Since $\hat{U}$ and $\hat{V}$ both have orthonormal columns, 
\begin{equation*}
\|R\|_2 \le \|E\|_2 \le \|E\|_F
\end{equation*}
\begin{equation*}
\|S\|_2 \le \|E\|_2 \le \|E\|_F
\end{equation*}
where 
\begin{equation*}
\|E\|_F \le c_{\sigma', \eta}\alpha_{sp}(\Phi)\sqrt\frac{r\bar{N}\log\bar{N}}{|\Omega|}\|\Phi\|_F 
\end{equation*}
By Weyl's inequality, 
\begin{equation*}
\sigma_r(\hat{\Phi}) \ge \sigma_r(\Phi) - \|E\|_2 \ge  \sigma_r(\Phi) - \|E\|_F
\end{equation*}
As a result, 
\begin{align*}
\angle(\hat{U}, U) &\le \frac{c_{\sigma', \eta}\alpha_{sp}(\Phi)\sqrt{\frac{r\bar{N}\log\bar{N}}{|\Omega|}}\|\Phi\|_F}{\sigma_r(\Phi) -c_{\sigma', \eta}\alpha_{sp}(\Phi)\sqrt{\frac{r\bar{N}\log\bar{N}}{|\Omega|}}\|\Phi\|_F} \land 1  \\
&\le \frac{c_{\sigma', \eta}\alpha_{sp}(\Phi)\sqrt{\frac{r\bar{N}\log\bar{N}}{|\Omega|}}r\|\Phi\|_2}{\sigma_r(\Phi) -c_{\sigma', \eta}\alpha_{sp}(\Phi)\sqrt{\frac{r\bar{N}\log\bar{N}}{|\Omega|}}r\|\Phi\|_2} \land 1  \\
&\le  \frac{c_{\sigma', \eta}\alpha_{sp}(\Phi)\sqrt\frac{r^3\bar{N}\log\bar{N}}{|\Omega|}}{\frac{\sigma_r(\Phi)}{\sigma_1(\Phi)} - c_{\sigma', \eta}\alpha_{sp}(\Phi)\sqrt\frac{r^3\bar{N}\log\bar{N}}{|\Omega|}}\land 1 
\end{align*}

\end{proof}

\begin{theorem}[Wedin's Theorem] \label{theorem: wedin}
Suppose that $X = U\Sigma V^\top$ is of rank $r$ and $\hat{X} = X + E$ with the leading $r$ left singular vector matrix and right singular vector matrix being $\hat{U}$ and $\hat{V}$. Then
\begin{equation*}
\max \{\angle(\hat{U}, U), \angle(\hat{V}, V) \} \le \frac{\max\{\|R\|_2, \|S\|_2\}}{\sigma_r(\hat{X})} \land 1
\end{equation*}
where 
\begin{align*}
R = X\hat{V} - \hat{U}\hat{\Sigma} &= -E\hat{V} \\
S = X'\hat{U} - \hat{V}\hat{\Sigma} &= -E'\hat{U} \\
\end{align*}
\end{theorem}

\begin{lemma} \label{lemma: quadratic}
Suppose that $T_i$ is almost surely not a linear combination of $U_i$. Under Assumption \ref{assump: confound}, $\frac{1}{N}T^\top(I - P_U)T$ is almost surely bounded away from 0 for any $N$.
\end{lemma}
\begin{proof}
Consider the asymptotic case when $N \to \infty$.
\begin{equation*}
\frac{1}{N}T^\top(I - P_U)T = \frac{1}{N}\sum_{i=1}^{N}T_i^2 - \frac{1}{N}\sum_{i=1}^{N}T_iU^\top_i(\frac{1}{N}\sum_{i=1}^{N}U_iU^\top_i)^{-1}\frac{1}{N}\sum_{i=1}^{N}U_iT_i
\end{equation*}
By Law of Large Number,
\begin{align*}
\frac{1}{N}\sum_{i=1}^{N}T_i^2 &\to \mathbb{E}(T_i^2) \\
\frac{1}{N}\sum_{i=1}^{N}T_iU^\top_i(\frac{1}{N}\sum_{i=1}^{N}U_iU^\top_i)^{-1}\frac{1}{N}\sum_{i=1}^{N}T_iU_i &\to \mathbb{E}(T_iU^\top_i)[\mathbb{E}{U_iU^\top_i}]^{-1}\mathbb{E}(U_iT_i)
\end{align*}
The result follows immediately from a matrix version of Cauchy-Schwartz Inequality \citep{tripathi1999matrix}.
\end{proof}

\begin{lemma} \label{lemma: bound_U} 
Under Assumption \ref{assump: confound}, $\frac{1}{\sqrt{Nr}}\|U\|$ is bounded for any $N$ with high probability at least $1 - 2\exp(-cN^{1/2})$.
\end{lemma}
\begin{proof}
Apply Theorem 5.39 in \citep{vershynin2010introduction} to matrix $L^{-1}U$, for any $t > 0$ and positive constants $c, C$, with probability at least $1 - 2\exp(-ct^2)$,
\begin{equation*}
\frac{1}{\sqrt{Nr}}\|U\| \le \frac{1}{\sqrt{Nr}}\|UL^{-1}L\| \le (1 + C\frac{\sqrt{r}}{\sqrt{N}} + \frac{t}{\sqrt{N}})\frac{1}{\sqrt{r}}\|L\|
\end{equation*}
Take $t = N^{1/4}$ then the conclusion follows.

\end{proof}

\begin{lemma}[Spikeness Ratio] \label{lemma: spikeness}
Under Assumption \ref{assump: confound}, the spikeness ratio $\alpha_{sp}(\Phi) \le c'c_V\sqrt{\bar{r}}$  with high probability $1 - N^{-1/2} - 2\exp(-cN^{1/2})$ for some positive constant $c, c'$.
\end{lemma}
\begin{proof}
According to the definition, $\alpha_{sp}(UV^\top) = \sqrt{Np}\frac{\|UV^\top\|_{\max}}{\|UV^\top\|_F}$. Obviously, 
\begin{equation*}
\|UV^\top\|_{\max} \le \max_{ij}(U_i^TV_j) \le \max_i\|U_i\|\max_{j}\|V_j\|
\end{equation*}
Next, we prove that $\|UV^\top\|_F \ge \sigma_r(U)\|V\|_F$. Suppose $U$ has SVD $\bar{U}_{n \times r}\bar{\Sigma}_{r \times r}\bar{V}^T_{r \times r}$ where $\bar{\Sigma}_{r \times r} = \operatorname{diag}(\sigma_1(U), ..., \sigma_r(U))$ and $\bar{V}^T\bar{V} = \bar{V}\bar{V}^T = I_{r \times r}$. Then 
\begin{align*}
\|UV\|_F &= \|\bar{U}_{n \times r}\bar{\Sigma}_{r \times r}\bar{V}^T_{r \times r}V\|_F \\
&= \|\begin{bmatrix} \sigma_1(U) & 0 & \dots & 0 \\ 
0 & \sigma_2(U) & \dots & 0 \\ 
\vdots & \vdots & \ddots & \vdots \\ 
0 & 0 & \dots & \sigma_r(U) \\ 
\end{bmatrix}\begin{bmatrix}
\bar{V}_1^TV \\ \bar{V}_2^TV \\ \vdots \\ \bar{V}_r^TV
\end{bmatrix}\|_F  \\
&= \sqrt{\sum_{k = 1}^{r}\|\sigma_k(U)\bar{V}_k^TV\|^2} \\ 
&\ge \sigma_r(U)\sqrt{\sum_{k = 1}^{r}\|\bar{V}_k^TV\|^2} \\
&= \sigma_r(U)\|\bar{V}^TV\|_F = \sigma_r(U)\|V\|_F
\end{align*}
Therefore, $\alpha_{sp}(UV^\top) \le \sqrt{Np}\frac{\frac{1}{\sqrt{N}}\max_i\|U_i\|\max_{j}\|V_j\|}{\frac{1}{\sqrt{N}}\sigma_r(U)\|V\|_F}$. 
Following the similar proof in Lemma \ref{lemma: bound_U}, we can prove that the following holds with high probability at least $1 - 2\exp(-cN^{1/2})$:
\begin{equation*}
\frac{1}{\sqrt{N}}\sigma_r(U) \ge (1 - C\sqrt{\frac{r}{N}} - \frac{1}{N^{1/4}})\|L\|
\end{equation*}
Under Assumption \ref{assump: confound}, $\|L^{-1}U_i\|^2 \sim \chi^2(r)$. Then according to Proposition 1 in \citep{hsu2012tail}, with probability at least $1 - \exp(-t^2/2)$, $\|U_i\| \le \sqrt{r + t\sqrt{2r} + t^2} \le \sqrt{r} + t$. Let $t = \sqrt{3\log N}$ and take union bound over $i = 1, ..., N$, then with high probability $1 - N^{-1/2}$ for any $i$, 
\begin{equation*}
\frac{1}{\sqrt{N}}\|U_i\| \le (\frac{\sqrt{r}}{\sqrt{N}} +  \frac{\sqrt{3\log N}}{\sqrt{N}})\|L\|
\end{equation*}
which implies that $\frac{1}{\sqrt{N}}\max_i\|U_i\| \le  (\frac{\sqrt{r}}{\sqrt{N}} +  \frac{\sqrt{3\log N}}{\sqrt{N}})\|L\|$.

Therefore, with high probability $1 - N^{-1/2} - 2\exp(-cN^{1/2})$,
\begin{equation*}
\alpha_{sp}(UV^\top) \le \frac{\sqrt{r} + \sqrt{3\log N}}{1 - C\frac{\sqrt{r}}{\sqrt{N}} - \frac{1}{N^{1/4}}}\sqrt{p}\frac{\max_{j}\|V_j\|}{\|V\|_F} \le c'c_V\sqrt{\bar{r}}
\end{equation*}
\end{proof}

\begin{lemma} \label{lemma: sv_ratio}
Under Assumption \ref{assump: confound}, $\frac{\sigma_r(\Phi)}{\sigma_1(\Phi)} \ge \sqrt{\frac{\underline{v}}{\underline{v} + 2\overline{v}}}$ with high probability $1 - 2\exp(-Cp^{\delta})$ given that $p^{1 + \delta}/N \to 0$ for some positive constant $\delta, C$.
\end{lemma}
\begin{proof}
We aim to prove $|x^\top(\frac{1}{N}V U^\top UV^\top - V L^\top L V^\top)x| \le \epsilon$ for any $x$ on the $p$-dimensional unit sphere $\mathcal{S}^{p-1}$. Since  $x^\top(\frac{1}{N}VU^\top UV^\top - VL^\top L V^\top)x = 0$ for $x \in \operatorname{Null}(V)$, we only have to prove
\begin{equation*}
\max_{x \in \mathcal{S}^{p-1} \cap \operatorname{Null}^\perp(V)} |x^\top(\frac{1}{N}VU^\top UV^\top - V L^\top L V^\top)x| \le \epsilon
\end{equation*}
where $S^{p-1} \cap \operatorname{Null}^\perp(V)$ is a $r$-dimensional space.

Consider $\frac{1}{4}$-net $\mathcal{N}$ for $\mathcal{S}^{D-1} \cap \operatorname{Null}^\perp(V)$, according to Lemma 5.4 in \citep{vershynin2010introduction}, 
\begin{equation*}
\max_{x \in \mathcal{S}^{p-1} \cap \operatorname{Null}^\perp(V)} |x^\top(\frac{1}{N}VU^\top UV^\top - V L^\top L V^\top)x| \le 2 \max_{x \in \mathcal{N}}|x^\top \frac{1}{N}V U^\top U V^\top x - x^\top V L^\top L V^\top x|
\end{equation*}
So we only need to prove that $\max_{x \in \mathcal{N}}|x^\top \frac{1}{N}V U^\top U V^\top x - x^\top V L^\top L V^\top x| \le \frac{\epsilon}{2}$ with high probability. Note that $\frac{1}{N}x^\top V U^\top U V^\top x - x^\top V L^\top L V^\top x = \frac{1}{N}\sum_i (Z^2_i - \mathbb{E}(Z_i^2))$, where $Z_i = U_iV^\top x$ are mutually independent with $\mathbb{E}(Z_i) = 0$ and $\mathbb{E}(Z_i^2) = x^TV L^\top L V^\top x \le \|VL^\top\|^2$. It follows that the $Z_i^2 - \mathbb{E}Z_i^2$ are sub-Exponential with upper bounded sub-Exponential norm (Lemma 5.14 \citep{vershynin2010introduction}):
\begin{equation*}
\|Z_i^2 - \mathbb{E}Z_i^2\| \le \|Z_i^2\|_{\psi_1} + \mathbb{E}Z_i^2 \le 2\|Z_i\|^2_{\psi_2} + \mathbb{E}Z_i^2 \le 3\|VL\|^2
\end{equation*}
By the Berstein Inequality (Corollary 5.17 in \citep{vershynin2010introduction})
\begin{equation*}
\mathbb{P}(|x'(\frac{1}{N}VU^\top UV^\top - V L^\top L V^\top)x| \ge \frac{\epsilon}{2}) \le 2\operatorname{exp}(-c\min\{\frac{\epsilon}{6\|VL\|}, \frac{\epsilon^2}{36\|VL\|^2}\}N)
\end{equation*}
Furthermore, Lemma 5.2 in \citep{vershynin2010introduction} implies that $|\mathcal{N}| \le 9^r$. So taking union bound over $\mathcal{N}$ gives:
\begin{equation*}
\mathbb{P}(\max_{x \in \mathcal{N}}|x'(\frac{1}{N}VU^\top UV^\top - V L^\top L V^\top)x| \ge \frac{\epsilon}{2}) \le 2\operatorname{exp}(r\log9-\frac{c}{\tilde{K}}\min\{{\epsilon}, {\epsilon^2}\}N)
\end{equation*}
where $\tilde{K}^{-1} = \min\{\frac{1}{6\|VL\|^2}, \frac{1}{36\|VL\|^4}\}$. 

We consider two cases: 
\begin{enumerate}
\item For large enough $p$ ($6 \underline{v} p > 1$), take $\epsilon = \frac{p^{2+\delta}}{N}$, then for some positive constant $C$ and $r/p^{\delta} \to 0$,
\begin{equation*}
\mathbb{P}(\max_{x \in \mathcal{N}}|x^\top (\frac{1}{N}VU^\top UV^\top - V L^\top L V^\top)x| \ge \frac{\epsilon}{2}) \le 2\exp(r\log9-\frac{c}{36\underline{v}^2p^2}\epsilon N) \le 2\operatorname{exp}(-Cp^{\delta})
\end{equation*}
So with probability at least $1 -  2\operatorname{exp}(-Cp^{\delta})$,
\begin{equation*}
\frac{\sigma_r^2(UV^\top)}{\sigma_1^2(UV^\top)} \ge \frac{\sigma_r^2(VL^\top) -\epsilon}{\sigma_1^2(VL^\top) + \epsilon} \ge \frac{\underline{v} - \epsilon/p}{\overline{v} + \epsilon/p} = \frac{\underline{v} - \frac{p^{1+\delta}}{N}}{\overline{v} + \frac{p^{1+\delta}}{N}} \ge \frac{\underline{v}}{2\overline{v} + \underline{v}}
\end{equation*}
which is bounded away from $0$ for large enough $N, p$ such that $\frac{p^{1+\delta}}{N} \le \frac{\underline{v}}{2}$.
\item For moderate $p$ ($6\underline{v} p \le 1$), take $\epsilon = \frac{p^{1/2+\delta/2}}{N}$ and then 
\begin{equation*}
\mathbb{P}(\max_{x \in \mathcal{N}}|x'(\frac{1}{N}V^TU^TUV - V^TV)x| \ge \frac{\epsilon}{2}) \le 2\operatorname{exp}(r\log9-\frac{c}{6\underline{v}p}{\epsilon^2}N) \le 2\operatorname{exp}(-Cp^{\delta}),
\end{equation*}
which implies the same conclusion.
\end{enumerate}
\end{proof}

\newpage
\section{More Numerical Results}
\begin{figure}[h] \label{figure: higherdim}
  \centering
  \includegraphics[width = \textwidth]{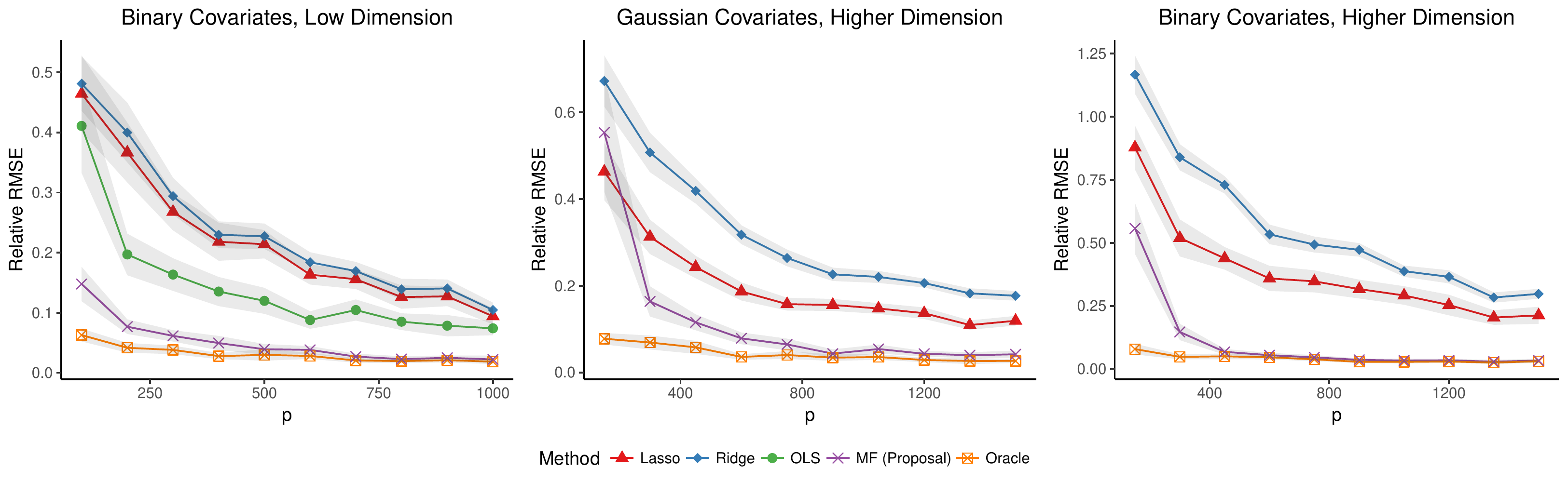}
  \caption{Relative RMSE for binary covariates in the low dimensional setting as in Section 4.1 and the relative RMSE for the setting where $p$ varies from $150$ to $1500$ and $N = p/1.5$. }
\end{figure}

\begin{figure}[h!] \label{figure: missing}
  \centering
  \includegraphics[width = \textwidth]{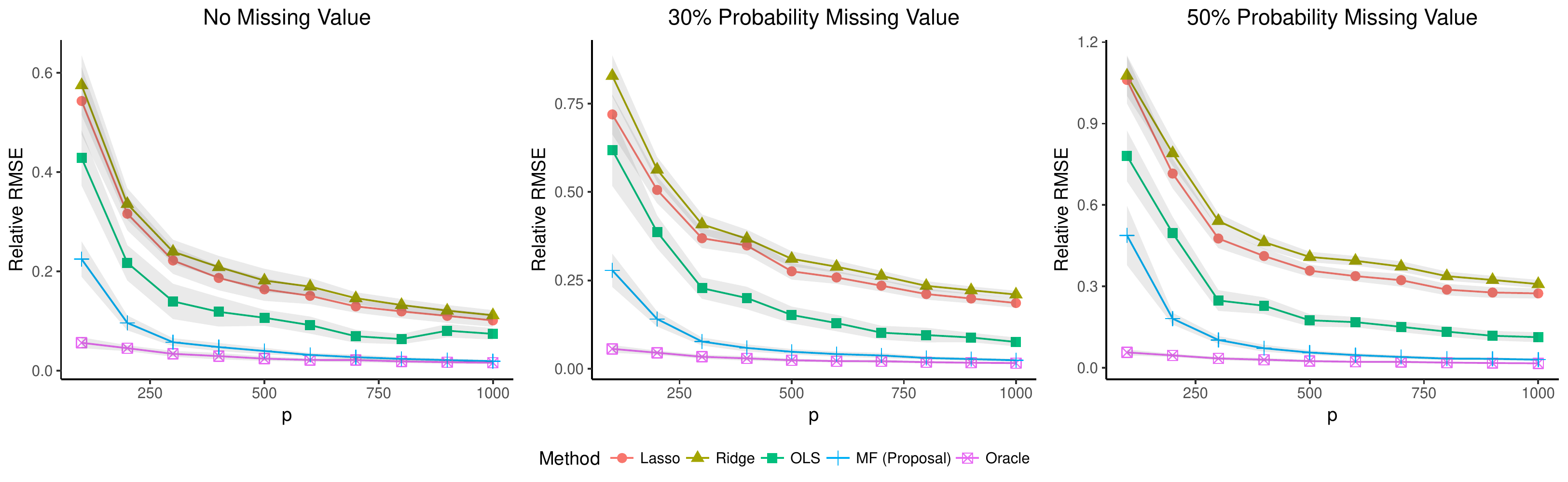}
  \caption{Relative RMSE of ATE estimators for binary covariates with $N = 200, 400, \dots, 2000$ and $p = N/2$. Each entry is set to be missing value with equal probability $0$, $0.3$, or $0.5$. }
\end{figure}

\begin{figure}[h!] \label{figure: fixp}
  \centering
  \includegraphics[height = 4.5cm, width = 10cm]{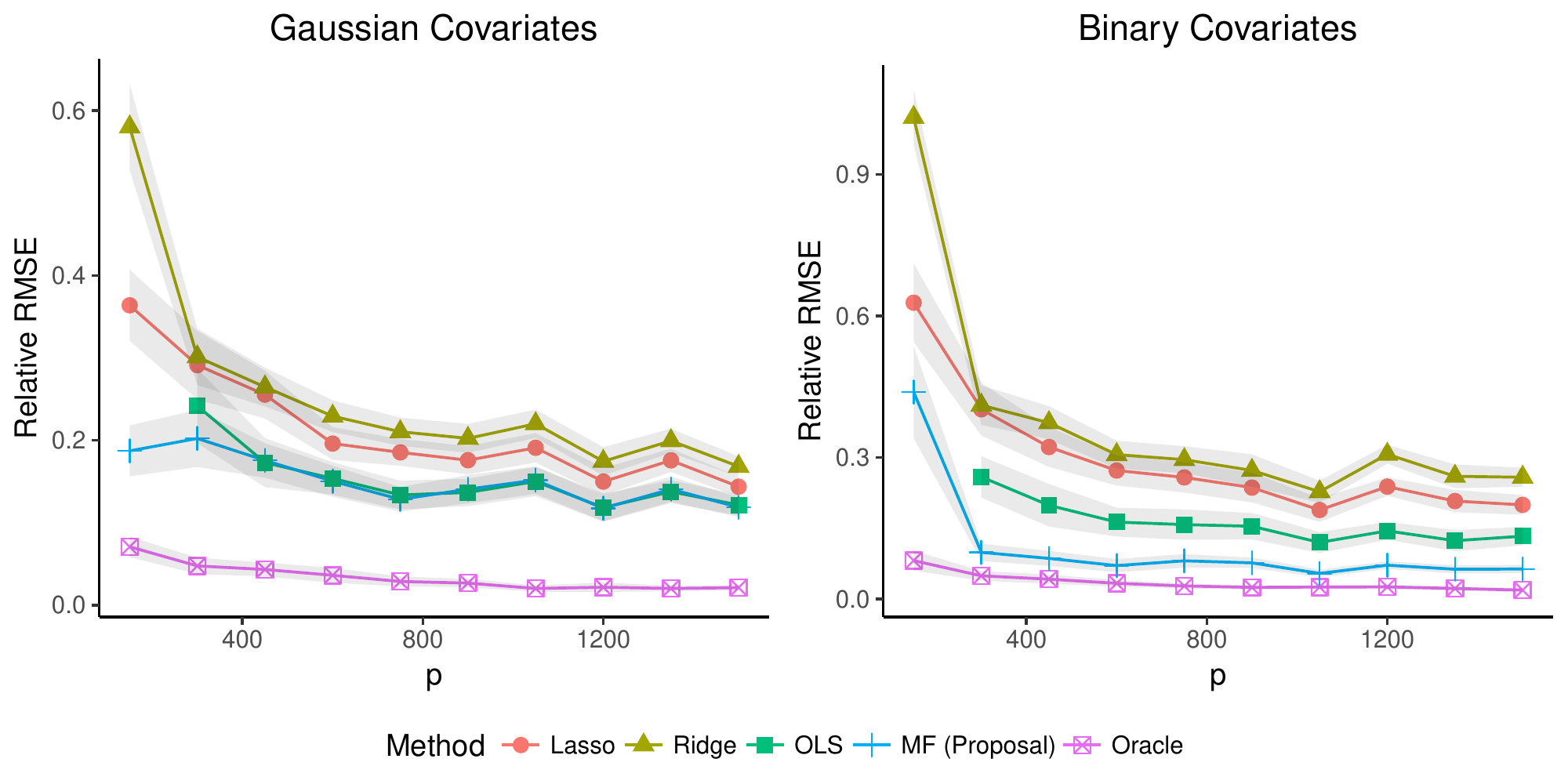}
  \caption{Relative RMSE of ATE estimators for Gaussian and Binary covariates with $N = 150, 300, \dots, 1500$ and $p = 200$. }
\end{figure}

\end{document}